\documentclass[11pt,journal,draftcls,letterpaper,onecolumn]{IEEEtran}

\def\U{\mathcal{U}}
\def\C{\mathbb{C}}
\def\R{\mathbb{R}}

\def\W{\mathbf{W}}

\usepackage{amsmath}
\usepackage{amsthm}
\usepackage{amsfonts}
\usepackage{comment}
\usepackage[normalem]{ulem} 
\usepackage[linesnumbered, ruled]{algorithm2e}
\usepackage{subcaption}

\SetKwRepeat{Do}{do}{while}
\usepackage[pdftex]{color}
\newtheorem{theorem}{Theorem}
\theoremstyle{plain}

\newtheorem{corollary}{Corollary}

\newtheorem{definition}{Definition}

\newtheorem{lemma}{Lemma}

\newtheorem{proposition}{Proposition}

\numberwithin{equation}{section}

\newcommand{\WS}{generic}

{\begin{list}{}%
         {\setlength{\leftmargin}{#1}}%
         \item[]%
}
{\end{list}}

\usepackage{graphicx}  

\definecolor{orange}{rgb}{1,0.5,0}

\usepackage{mathptmx}

\DeclareMathOperator{\rank}{rank}

\newcommand{\bracket}[1]{\langle #1 \rangle}

\begin{document}
\title{CUR Decompositions, Similarity Matrices, and Subspace Clustering}
\author{Akram Aldroubi, Keaton Hamm, Ahmet Bugra Koku, and Ali Sekmen}

\maketitle

\begin{abstract}
A general framework for solving the subspace clustering problem using the CUR decomposition is presented.  The CUR decomposition provides a natural way to construct  similarity matrices for data that come from a union of unknown subspaces $\U=\underset{i=1}{\overset{M}\bigcup}S_i$.  The similarity matrices thus constructed give the exact clustering in the noise-free case.  Additionally, this decomposition gives rise to many distinct similarity matrices from a given set of data, which allow enough flexibility to perform accurate clustering of noisy data.  We also show that two known methods for subspace clustering can be derived from the CUR decomposition. An algorithm based on the theoretical construction of similarity matrices is presented, and experiments on synthetic and real data are presented to test the method.

Additionally, an adaptation of our CUR based similarity matrices is utilized to provide a heuristic algorithm for subspace clustering; this algorithm yields the best overall performance to date for clustering the Hopkins155 motion segmentation dataset.

\end{abstract}


\section{Introduction}

We present here two tales: one about the so-called CUR decomposition (or sometimes skeleton decomposition), and another about the subspace clustering problem.  It turns out that there is a strong connection between the two subjects in that the CUR decomposition provides a general framework for the similarity matrix methods used to solve the subspace clustering problem, while also giving a natural link between these methods and other minimization problems related to subspace clustering.   

The CUR decomposition is remarkable in its simplicity as well as its beauty:  one can decompose a given matrix $A$ into the product of three matrices, $A=CU^\dagger R$, where $C$ is a subset of  columns  of $A$, $R$ is a subset of rows of $A$, and $U$ is their intersection (see Theorem \ref{CURdecomp} for a precise statement).  The primary uses of the CUR decomposition to date are in the field of scientific computing.  In particular, it has been used as a low-rank approximation method that is more faithful to the data structure than other factorizations \cite{Mahoney09,Boutsidis17}, an approximation to the singular value decomposition  \cite{Drineas06I,Drineas06II,Drineas06III}, and also has provided efficient algorithms to compute with and store massive matrices in memory.  In the sequel, it will be shown that this decomposition is the source of some well-known methods for solving the subspace clustering problem, while also adding the construction of many similarity matrices based on the data.

The subspace clustering problem may be stated as follows: suppose that some collected data vectors in $\mathbb{K}^m$ (with $m$ large, and $\mathbb{K}$ being either $\R$ or $\C$)  comes from a union of linear subspaces (typically low-dimensional) of $\mathbb{K}^m$, which will be denoted by $\U=\underset{i=1}{\overset{M}\bigcup}S_i$.  However, one does not know {\em a priori} what the subspaces are, or even how many of them there are.  Consequently, one desires to determine the number of subspaces represented by the data, the dimension of each subspace, a basis for each subspace, and finally to cluster the data: the data $\{w_j\}_{j=1}^n\subset\U$ are not ordered in any particular way, and so clustering the data means to determine which data belong to the same subspace.

There are indeed physical systems which fit into the model just described.  Two particular examples are motion tracking and facial recognition.  For example, the Yale Face Database B \cite{Yale} contains images of faces, each taken with 64 different illumination patterns.  Given a particular subject $i$,  there are 64 images of their face illuminated differently, and each image represents a vector lying approximately in a low-dimensional linear subspace, $S_i$, of the higher dimensional space $\R^{307,200}$ (based on the size of the greyscale images).  It has been experimentally shown that images of a given subject approximately lie in a subspace $S_i$ having dimension 9 \cite{Basri01}.  Consequently, a data matrix obtained from facial images under different illuminations has columns which lie in the union of low-dimensional subspaces, and one would desire an algorithm which can sort, or cluster, the data, thus recognizing which faces are the same.

There are many avenues of attack to the subspace clustering problem, including iterative and statistical methods \cite{Kanatani03, Akram09,Akram08,Tseng00,Bolles81,Silva08,Zhang10,Zhang12}, algebraic methods \cite{Vidal05, Ma08, Vidal17}, sparsity methods \cite{Eldar09,Elhamifar09,Elhamifar10,Liu10}, minimization problem methods inspired by compressed sensing \cite{Liu10, Liu10_2}, and methods based on spectral clustering \cite{Elhamifar09,Elhamifar10,Luxburg07, Lerman09_3,Schnorr09,Yan06,Vidal07_2, Lerman09_4}.  For a thorough, though now incomplete, survey on the spectral clustering problem, the reader is invited to consult \cite{Vidal10}.

Some of the methods mentioned above begin by finding a {\em similarity matrix} for a given set of data, i.e. a square matrix whose entries are nonzero precisely when the corresponding data vectors lie in the same subspace, $S_i$, of $\mathcal{U}$ (see Definition \ref{DefSIM} for the precise definition).  The present article is concerned with a certain matrix factorization method -- the CUR decomposition -- which provides a quite general framework for finding a similarity matrix for data that fits the subspace model above.  It will be demonstrated that the CUR decomposition indeed produces many similarity matrices for subspace data.  Moreover, this decomposition provides a bridge between matrix factorization methods and the minimization problem methods such as Low-Rank Representation \cite{Liu10, Liu10_2}.

\subsection{Paper Contributions}

\begin{itemize}
\item  In this work, we show that the CUR decomposition gives rise to  similarity matrices for clustering data that comes from a union of independent subspaces.  Specifically, given the data matrix $\W=[w_1 \dotsb w_n] \subset \mathbb{K}^m$  drawn from a union $\mathcal{U}=\bigcup_{i=1}^{M}S_i$ of independent subspaces $\left\{S_i\right\}_{i=1}^{M}$ of  dimensions $\left\{d_i\right\}_{i=1}^{M}$, any CUR decomposition $\W = C U^{\dagger} R$  can be used to construct a similarity matrix for $\W$. In particular, if $Y=U^\dagger R$ and $Q$ is the element-wise binary or absolute value version of $Y^*Y$, then $\Xi_{\W} = Q^{d_{max}}$ is a similarity matrix for $\W$; i.e. $\Xi_{\W}(i,j) \neq 0$ if the columns $w_i$ and $w_j$ of $\W$ come from the same subspace, and $\Xi_{\W}(i,j) = 0$ if the columns $w_i$ and $w_j$ of $\W$ come from different subspaces.

\item This paper extends our previous framework for finding similarity matrices for clustering data that comes from the union of independent subspaces. In \cite{Akram16}, we showed that  any factorization $\W = B P$, where the columns of $B$ come from $\mathcal{U}$ and form a basis for the column space of $\W$, can be used to produce a similarity matrix $\Xi_{\W}$. This work shows that we do not need to limit the factorization of $\W$ to bases, but may extend it to frames, thus allowing for more flexibility.

\item Starting from the CUR decomposition framework, we demonstrate that some well-known methods utilized in subspace clustering follow as special cases, or are tied directly to the CUR decomposition; these methods include the \textit{Shape Interaction Matrix} \cite{Costeira98,Ji15} and \textit{Low-Rank Representation \cite{Liu10, Liu10_2}}.

\item  A proto-algorithm is presented which modifies the similarity matrix construction mentioned above to allow clustering of noisy subspace data.  Experiments are then conducted on synthetic and real data (specifically, the Hopkins155 motion dataset) to justify the proposed theoretical framework.  It is demonstrated that using an average of several CUR decompositions to find similarity matrices for a data matrix $\W$ outperforms many known methods in the literature while being computationally fast.

\item A clustering algorithm based on the methodology of the {\em Robust Shape Interaction Matrix} of Ji, Salzmann, and Li \cite{Ji15} is also considered, and using our CUR decomposition framework together with their algorithm yields the best performance to date for clustering the Hopkins155 motion dataset.
\end{itemize}

\subsection{Layout}

The rest of the paper develops as follows: a brief section on preliminaries is followed by the statement and discussion of the most general exact CUR decomposition.  Section \ref{SECCURSeg} contains the statements of the main results of the paper, while Section \ref{SECCases} contains the relation of the general framework that CUR gives for solving the subspace clustering problem.  The proofs of the main theorems are enumerated in Section \ref{SECProofs} followed by our algorithm and numerical experiments in Section \ref{SECExperiments}, whereupon the paper concludes with some discussion of future work.  

\section{Preliminaries}
\subsection{Definitions and Basic Facts}

Throughout the sequel, $\mathbb{K}$ will refer to either the real or complex field ($\R$ or $\C$, respectively).   For $A\in\mathbb{K}^{m\times n}$, its {\em Moore--Penrose psuedoinverse} is the unique matrix $A^\dagger\in\mathbb{K}^{n\times m}$ which satisfies the following conditions: 
\begin {enumerate} \item  $AA^\dagger A = A,$ 
\item $A^\dagger AA^\dagger = A^\dagger,$ 
\item $(AA^\dagger)^*=AA^\dagger,$ and 
\item $(A^\dagger A)^*=A^\dagger A.$ 
\end{enumerate}
 Additionally, if $A=U\Sigma V^*$ is the Singular Value Decomposition of $A$, then $A^\dagger=V\Sigma^\dagger U^*$, where the  pseudoinverse of  the $m \times n$ matrix $\Sigma = \text{diag}(\sigma_1,\dots,\sigma_r,0,\dots,0)$ is the $n\times m$ matrix $\Sigma^\dagger=\text{diag}(1/\sigma_1,\dots,1/\sigma_r, 0\dots,0)$.  For these and related notions, see Section 5 of \cite{GolubVanLoan}.

Also of utility to our analysis is that a rank $r$ matrix has a so-called {\em skinny SVD} of the form $A=U_r\Sigma_r V_r^*$, where $U_r$ comprises the first $r$ left singular vectors of $A$, $V_r$ comprises the first $r$ right singular vectors, and $\Sigma_r = \text{diag}(\sigma_1,\dots,\sigma_r)\in\mathbb{K}^{r\times r}$.  Note that in the case that $\rank(A)>r$, the skinny SVD is simply a low-rank {\em approximation} of $A$.

\begin{definition}[Independent Subspaces]
\label{indp:subspaces}
Non-trivial subspaces $\{S_i \subset \mathbb{K}^m\}_{i=1}^{M}$ are called {\em independent} if their dimensions satisfy the following relationship:
\[dim(S_1+\dots+S_M) = dim(S_1)+\dots+dim(S_M) \leq m.\]
\end{definition}
The definition above is equivalent to the property  that any set of non-zero vectors $\{w_1,\dotsb,w_M\}$ such that $w_i\in S_i$, $i=1,\dots,M$ is linearly independent.

\begin{definition}[Generic Data]
 Let $S$ be a  linear subspace of $\mathbb{K}^{m}$ with dimension $d$. A set of data $\W$ drawn from $S$ is said to be {\em \WS} if  (i) $|\W|>d$, and (ii)  every  $d$ vectors from  $\W$  form a basis for $S$.
\end{definition}

Note this definition is equivalent to the frame-theoretic description that the columns of $\W$ are a frame for $S$ with spark $d+1$ (see \cite{Alexeev12,Donoho03PNAS}).  It is also sometimes said that the data $\W$ is in {\em general position}.

\begin{definition}[Similarity Matrix] \label{DefSIM}
Suppose $\W=[w_1 \dotsb w_n] \subset \mathbb{K}^m$  has columns drawn from a union of subspaces $\mathcal{U}=\bigcup_{i=1}^{M}S_i$. We say $\Xi_{\W}$ is a {\em similarity matrix for} $\W$ if and only if (i) $\Xi_{\W}$ is symmetric, and (ii) $\Xi_{\W}(i,j) \neq 0$ if and only if $w_i$ and $w_j$ come from the same subspace.
\end{definition}

Finally, if $A\in\mathbb{K}^{m\times n}$, we define its {\em absolute value version} via $\text{abs}(A)(i,j)=|A(i,j)|$, and its {\em binary version} via  $\text{bin}(A)(i,j) = 1$ if $A(i,j)\neq0$ and $\text{bin}(A)(i,j)=0$ if $A(i,j)=0.$

\subsection{Assumptions}\label{SECAssumptions}

 In the rest of what follows, we will assume that $\mathcal{U}=\bigcup_{i=1}^{M}S_i$ is a nonlinear set consisting of the union of non-trivial, independent, linear subspaces $\left\{S_i\right\}_{i=1}^{M}$ of $\mathbb{K}^{m}$, with corresponding  dimensions $\left\{d_i\right\}_{i=1}^{M}$, with $d_{max} := \max_{1\leq i\leq M}d_i$. We will assume that the data matrix $\W=[w_1 \dotsb w_n] \in \mathbb{K}^{m \times n}$ has column vectors that are drawn  from $\mathcal{U}$, and that the data drawn from each subspace $S_i$ is generic for that subspace.

\section{CUR decomposition}

Our first tale is the  remarkable CUR matrix decomposition, also  known as the skeleton decomposition \cite{Goreinov97,Demanet13} whose proof can be obtained by basic linear algebra. 
\begin{theorem} 
\label{CURdecomp}Suppose  $A \in \mathbb{\mathbb K}^{m\times n}$ has rank $r$. Let $I\subset\{1,\dots,m\}$, $J\subset \{1,\dots,n\}$ with $|I|=s$ and $|J|=k$, and let $C$  be the $m\times k$ matrix whose columns are  the columns of $A$ indexed by $J$.  Let $R$ be the $s\times n$ matrix whose rows are  the rows of $A$ indexed by $I$. Let $U$ be the $s\times k$ sub-matrix of $A$ whose entries are indexed by $I\times J$. If $\rank(U)=r$, then $A=CU^\dagger R$. 
\end{theorem}

\begin{proof} Since $U$ has rank $r$, $\rank (C)= r$.  Thus the columns of $C$ form a frame for the column space of $A$, and we have $A=CX$ for some (not necessarily unique) $k\times n$ matrix $X$. Let $P_I$ be an $s\times m$ row selection matrix such that $R=P_IA$; then we have $R=P_IA=P_IC X$.  Note also that $U=P_IC$, so that the last equation can then be written as $R=UX$. Since $\rank(R)=r$, any solution to $R=UX$ is also a solution to $A=CX$. Thus the conclusion of the theorem follows upon observing that $Y=U^\dagger R$ is a solution to $R=UX$. Indeed, the same argument as above implies that $U^\dagger R$ is a solution to $R=UX$ if and only if it is a solution to $RP_J=U = UXP_J$ where $P_J$ is a $n\times k$ column-selection matrix which picks out columns according to the index set $J$.  Thus, noting that $UU^\dagger RP_J = UU^\dagger U = U$ completes the proof, whence $A=CY=CU^\dagger R$.
\end{proof} 

Note that the assumption on the rank of $U$ implies that $k,s\geq r$ in the theorem above.  While this theorem is quite general, it should be noted that in some special cases, it reduces to a much simpler decomposition, a fact that is recorded in the following corollaries.  The proof of each corollary follows from the fact that the pseudoinverse $U^\dagger$ takes those particular forms whenever the columns or rows are linearly independent (\cite[p. 257]{GolubVanLoan}, for example).

\begin{corollary}
Let $A, C, U$, and $R$ be as in Theorem \ref{CURdecomp} with  $C\in\mathbb{K}^{m\times r}$; in particular, the columns of $C$ are linearly independent.  Then $A=C(U^*U)^{-1}U^*R.$
\end{corollary}

\begin{corollary}
Let $A, C, U,$ and $R$ be as in Theorem \ref{CURdecomp} with $R\in\mathbb{K}^{r\times n}$; in particular, the rows of $R$ are linearly independent.  Then $A = CU^*(UU^*)^{-1}R.$
\end{corollary}

\begin{corollary}\label{CORCUR}
Let $A, C, U,$ and $R$ be as in Theorem \ref{CURdecomp} with $U\in\mathbb{K}^{r\times r}$; in particular, $U$ is invertible.  Then $A = CU^{-1}R.$
\end{corollary}

In most sources, the decomposition of Corollary \ref{CORCUR} is what is called the skeleton or CUR decomposition, \cite{Goreinov97_2}, though the case when $k=s>r$ has been treated in \cite{Caiafa10}.  The statement of Theorem \ref{CURdecomp} is the most general version of the exact CUR decomposition.

The precise history of the CUR decomposition is somewhat difficult to discern.  Many articles cite Gantmacher \cite{Gantmacher60}, though the authors have been unable to find the term skeleton decomposition therein.  However, the decomposition does appear implicitly (and without proof) in a paper of Penrose from 1955, \cite{Penrose}. However,  Perhaps the modern starting point of interest in this decomposition is the work of Goreinov, Tyrtyshnikov, and Zamarashkin \cite{Goreinov97,Goreinov97_2}.  They begin with the CUR decomposition as in Corollary \ref{CORCUR}, and study the error $\|A-CU^{-1}R\|_2$ in the case that $A$ has rank larger than $r$, whereby the decomposition $CU^{-1}R$ is only approximate.  Additionally, they allow more flexibility in the choice of $U$ since computing the inverse directly may be computationally difficult. See also \cite{Stewart99,Berry05,Wang13}.

More recently, there has been renewed interest in this decomposition.  In particular, Drineas, Kannan, and Mahoney \cite{Drineas06III} provide two randomized algorithms which compute an approximate CUR factorization of a given matrix $A$.  Moreover, they provide error bounds based upon the probabilistic method for choosing $C$ and $R$ from $A$.  It should also be noted that their middle matrix $U$ is not $U^\dagger$ as in Theorem \ref{CURdecomp}.  Moreover, Mahoney and Drineas \cite{Mahoney09} give another CUR algorithm based on a  way of selecting columns which provides nearly optimal error bounds for $\|A-CUR\|_F$ (in the sense that the optimal rank $r$ approximation to any matrix $A$ in the Frobenius norm is its skinny SVD of rank $r$, and they obtain error bounds of the form $\|A-CUR\|_F\leq(2+\varepsilon)\|A-U_r\Sigma_rV_r^T\|_F$).  They also note that the CUR decomposition should be favored in analyzing real data that is low dimensional because the matrices $C$ and $R$ maintain the structure of the data, and the CUR decomposition actually admits an viable interpretation of the data as opposed to attempting to interpret singular vectors of the data matrix, which are generally linear combinations of the data.  See also \cite{Drineas08}.

Subsequently, others have considered algorithms for computing CUR decompositions which still provide approximately optimal error bounds in the sense described above; see, for example, \cite{Boutsidis17,Demanet13,Voronin17,Wang16,Oswal17}.  For applications of the CUR decomposition in various aspects of data analysis across scientific disciplines, consult \cite{Li10CUR,Yip14,Yang15,Xu15}.  Finally, a very recent paper discusses a convex relaxation of the CUR decomposition and its relation to the joint learning problem \cite{CURTPAMI}.

 It should be noted that the CUR decomposition is one of a long line of matrix factorization methods, many of which take a similar form.  The main idea is to write $A=BX$ for some less complicated or more structured matrices $B$ and $X$.  In the case that $B$ consists of columns of $A$, this is called the \textit{interpolative decomposition}, of which CUR is a special case.  Alternative methods include the classical suspects such as $LU$, $QR$ and singular value decompositions.  For a thorough treatment of such matters, the reader is invited to consult the excellent survey \cite{HMT}.  In general, low rank matrix factorizations find utility in a broad variety of applications, including copyright security \cite{Muhammad15,Muhammad18Watermark}, imaging \cite{CandesMRI}, and matrix completion \cite{CandesCompletion} to name a very few.

\section{Subspace Clustering via CUR Decomposition}\label{SECCURSeg}

Our second tale is one of the utility of the CUR decomposition in the similarity matrix framework for solving the subspace segmentation problem discussed above.  Prior works have typically focused on CUR as a low-rank matrix approximation method which has a low cost, and also remains more faithful to the initial data than the singular value decomposition.  This perspective is quite useful, but here we demonstrate what appears to be the first application in which CUR is responsible for an overarching framework, namely subspace clustering.

As mentioned in the introduction, one approach to clustering subspace data is to find a similarity matrix from which one can simply read off the clusters, at least when the data exactly fits the model and is considered to have no noise.  The following theorem provides a way to find many similarity matrices for a given data matrix $\W$, all stemming from different CUR decompositions (recall that a matrix has very many CUR decompositions depending on which columns and rows are selected).

\begin{theorem}\label{theoremsimilarity}
 Let $\W=[w_1 \dotsb w_n] \in \mathbb{K}^{m \times n}$ be a rank $r$ matrix  whose columns are drawn from $\mathcal{U}$ which satisfies the assumptions in Section \ref{SECAssumptions}. Let $\W$ be factorized as $\W = C U^{\dagger} R$ where $C \in \mathbb{K}^{m\times k}$, $R \in \mathbb{K}^{s \times n}$, and $U\in\mathbb{K}^{s\times k}$ are as in Theorem \ref{CURdecomp}, and let $Y=U^\dagger R$ and $Q$ be either the binary or absolute value version of $Y^*Y$. Then, $\Xi_{\W} = Q^{d_{max}}$ is a similarity matrix for $\W$.
\end{theorem}

The key ingredient in the proof of Theorem \ref{theoremsimilarity} is the fact that the matrix $Y=U^\dagger R$, which generates the similarity matrix, has a block diagonal structure due to the independent subspace structure of $\mathcal{U}$; this fact is captured in the following theorem. 

\begin{theorem}
\label{theoremframes}
Let $\W, C, U,$ and $R$ be as in Theorem \ref{theoremsimilarity}.  If $Y=U^{\dagger}R$, then there exists a permutation matrix $P$ such that \[
YP =
\begin{bmatrix}
 Y_1 & 0 & \dotsb & 0\\
0 & Y_2 & \dotsb & 0\\
\vdots & \vdots & \ddots & \vdots\\
0 & \dotsb & 0 & Y_M\\
 \end{bmatrix},
 \]
where each $Y_i$ is a matrix of size $k_i\times n_i$, where $n_i$ is the number of columns in $\W$ from subspace $S_i$, and $k_i$ is the number of columns in $C$ from $S_i$.  Moreover, $\W P$ has the form $[\W_1\dots \W_M]$ where the columns of $\W_i$ come from the subspace $S_i$.
\end{theorem}

The proofs of the above facts will be related in a subsequent section. 

The role of $d_{\max}$ in Theorem \ref {theoremsimilarity} is that $Q$ is almost a similarity matrix but each cluster may not be fully connected. By raising  $Q$ to the power $d_{\max}$ we ensure that each cluster is fully connected.

The next section will demonstrate that some well-known solutions to the subspace clustering problem are consequences of the CUR decomposition.  For the moment, let us state some potential advantages that arise naturally from the statement of Theorem \ref{theoremsimilarity}.  One of the advantages of the CUR decomposition is that one can construct many similarity matrices for a data matrix $\W$ by choosing different representative rows and columns; i.e. choosing different matrices $C$ or $R$ will yield different, but valid similarity matrices.  A possible advantage of this is that for large matrices, one can reduce the computational load by choosing a comparatively small number of rows and columns.  Often, in obtaining real data, many entries may be missing or extremely corrupted.  In motion tracking, for example, it could be that some of the features are obscured from view for several frames.  Consequently, some form of matrix completion may be necessary.  On the other hand, a look at the CUR decomposition reveals that whole rows of a data matrix can be missing as long as we can still choose enough rows such that the resulting matrix $R$ has the same rank as $\W$.  
In real situations when the data matrix is noisy, then there is no exact CUR decomposition for $\W$; however, if the rank of the clean data $\W$ is well-estimated, then one can compute several CUR {\em approximations} of $\W$, i.e. if the clean data should be rank $r$, then approximate $\W\approx CU^\dagger R$ where $C$ and $R$ contain at least $r$ columns and rows, respectively.  From each of these approximations, an approximate similarity may be computed as in Theorem \ref{theoremsimilarity}, and some sort of averaging procedure can be performed to produce a better approximate similarity matrix for the noisy data.  This idea is explored more thoroughly in Section \ref{SECExperiments}.

\section{Special Cases}\label{SECCases}
\subsection{Shape Interaction Matrix}
\label{SIM}

In their pioneering work on factorization methods for motion tracking \cite{Costeira98}, Costeira and Kanade introduced the {\em Shape Interaction Matrix}, or SIM.  Given a data matrix $\W$ whose skinny SVD is $U_r\Sigma_r V_r^*$, SIM($\W$) is defined to be $V_rV_r^*$.  Following their work, this has found wide utility in theory and in practice.  Their observation was that the SIM often provides a similarity matrix for data coming from independent subspaces.  It should be noted that in \cite{Akram16}, it was shown that examples of data matrices $\W$ can be found such that  $V_rV_r^*$ is not a similarity matrix for $\W$; however, it was noted there that SIM($\W$) is almost surely a similarity matrix (in a sense made precise therein).

Perhaps the most important consequence of Theorem \ref{theoremsimilarity} is that the shape interaction matrix is a special case of the general framework of the CUR decomposition.  This fact is shown in the following two corollaries of Theorem \ref{theoremsimilarity}, whose proofs may be found in Section \ref{SECCor}

\begin{corollary}
\label{corollarySIM}
 Let $\W=[w_1 \dotsb w_n] \in \mathbb{K}^{m \times n}$ be a rank $r$ matrix  whose columns are drawn from $\mathcal{U}$. Let $\W$ be factorized as $\W = \W \W^{\dagger} \W$, and let $Q$ be either the binary or absolute value version of $\W^{\dagger}\W$.  Then $\Xi_{\W} = Q^{d_{max}}$ is a similarity matrix for $\W$.  Moreover, if the skinny singular value decomposition of $\W$ is $\W=U_r\Sigma_r V_r^*$, then $\W^\dagger \W=V_rV_r^*$.  
\end{corollary}

\begin{corollary}\label{corollaryRR}
Let $\W=[w_1 \dotsb w_n] \in \mathbb{K}^{m \times n}$ be a rank $r$ matrix  whose columns are drawn from $\mathcal{U}$. Choose $C=\W$, and $R$ to be any rank $r$ row restriction of $\W$.  Then $\W = \W R^{\dagger} R$ by Theorem \ref{CURdecomp}. Moreover, $R^\dagger R=\W^\dagger \W=V_rV_r^*$, where $V_r$ is as in Corollary \ref{corollarySIM}.
\end{corollary}

It follows from the previous two corollaries that in the ideal (i.e. noise-free case), the shape interaction matrix of Costeira and Kanade is a special case of the more general CUR decomposition.  However, note that $Q^{d_{max}}$ need not be the SIM, $V_rV_r^*$, in the more general case when $C\neq\W$.

\subsection{Low-Rank Representation Algorithm}
Another class of methods for solving the subspace clustering problem arises from solving some sort of minimization problem.  It has been noted that in many cases such methods are intimately related to some matrix factorization methods \cite{Vidal10,Wei11}.  

One particular instance of a minimization based algorithm is the {\em Low Rank Representation} (LRR) algorithm of Liu, Lin, and Yu \cite{Liu10, Liu10_2}.  As a starting point, the authors consider the following rank minimization problem:
\begin{equation}
\label{eqn:lrr}
\min_{Z}\; \rank (Z) \quad \textrm{s.t.}  \quad \W = AZ,
\end{equation}
where $A$ is a dictionary that linearly spans $\W$. 

Note that there is indeed something to minimize over here since if $A=\W$, $Z=I_{n\times n}$ satisfies the constraint, and evidently $\rank(Z)=n$; however, if $\rank(\W)=r$, then $Z=\W^\dagger\W$ is a solution to $\W=\W Z$, and it can be easily shown that $\rank(\W^\dagger\W)=r$.  Note further that any $Z$ satisfying $\W=\W Z$ must have rank at least $r$, and so we have the following.
\begin{proposition}
Let $\W$ be a rank $r$ data matrix whose columns are drawn from $\mathcal{U}$, then $\W^\dagger \W$ is a solution to the minimization problem 
\[\underset{Z}\min \;\rank(Z) \quad \textrm{s.t.} \quad \W=\W Z.\]
\end{proposition}
Note that in general, solving this minimization problem \eqref{eqn:lrr} is NP--hard (a special case of the results of \cite{RankMin08}; see also \cite{Recht08}).  Note that this problem is equivalent to minimizing $\|\sigma(Z)\|_0$ where $\sigma(Z)$ is the vector of singular values of $Z$, and $\|\cdot\|_0$ is the number of nonzero entries of a vector.  Additionally, the solution to Equation~\ref{eqn:lrr} is generally not unique, so typically the rank function is replaced with some norm to produce a convex optimization problem.  Based upon intuition from the compressed sensing literature, it is natural to consider replacing $\|\sigma(Z)\|_0$ by $\|\sigma(Z)\|_1$, which is the definition of the nuclear norm, denoted by $\|Z\|_*$ (also called the trace norm, Ky--Fan norm, or Shatten 1--norm).  In particular, in \cite{Liu10}, the following was considered:
\begin{equation}
\label{eqn:lrr:nuclear}
\min_{Z} \|Z\|_{*} \quad \textrm{s.t.}  \quad \W = AZ.
\end{equation}
Solving this minimization problem applied to subspace clustering is dubbed {\em Low-Rank Representation} by the authors in \cite{Liu10}.

Let us now specialize these problems to the case when the dictionary is chosen to be the whole data matrix, in which case we have
\begin{equation}
\label{eqn:lrr:nuclear:self}
\min_{Z} \|Z\|_{*} \quad \textrm{s.t.}  \quad \W = \W Z.
\end{equation}
It was shown in \cite{Liu10_2, Wei11} that the SIM defined in Section~\ref{SIM}, is the unique solution to problem \eqref{eqn:lrr:nuclear:self}:

\begin{theorem}[\cite{Wei11}, Theorem 3.1]\label{THMWeiLin}
Let $\W\in\mathbb{K}^{m\times n}$ be a rank $r$ matrix whose columns are drawn from $\mathcal{U}$, and let $\W=U_r\Sigma_r V_r^*$ be its skinny SVD.  Then $V_rV_r^*$ is the unique solution to \eqref{eqn:lrr:nuclear:self}.
\end{theorem}

For clarity and completeness of exposition, we supply a simpler proof of Theorem \ref{THMWeiLin} here than appears in \cite{Wei11}. 

\begin{proof}
Suppose that $\W=U\Sigma V^*$ is the full SVD of $\W$.  Then since $\W$ has rank $r$, we can write
\[\W = U\Sigma V^* = \begin{bmatrix} U_r & \tilde U_r\end{bmatrix} \begin{bmatrix} \Sigma_r & 0\\ 0 & 0\\\end{bmatrix}\begin{bmatrix} V_r^* \\ \tilde V_r^*\end{bmatrix},\]
where $U_r\Sigma_rV_r^*$ is the skinny SVD of $\W$.  
Then if $\W=\W Z$, we have that $I-Z$ is in the null space of $\W$.  Consequently, $I-Z=\tilde V_r X$ for some $X$.  Thus we find that 
\begin{align*}
Z & = I+\tilde V_r X\\
& = VV^*+\tilde V_r X\\
& = V_rV_r^*+\tilde V_r\tilde V_r^*+\tilde V_r X\\
& = V_r V_r^*+\tilde V_r(\tilde V_r^*+X)\\
& =:A+B.
\end{align*}

Recall that the nuclear norm is unchanged by multiplication on the left or right by unitary matrices, whereby we see that $\|Z\|_*=\|V^*Z\|_*=\|V^*A+V^*B\|_*$.  However, 
\[V^*A +V^* B= \begin{bmatrix} V_r^*\\ 0 \end{bmatrix}+\begin{bmatrix}0 \\ \tilde V_r^*+X\end{bmatrix}.\]
Due to the above structure, we have that $\|V^*A+V^*B\|_*\geq\|V^*A\|_*,$ with equality if and only if $V^*B=0$ (for example, via the same proof as \cite[Lemma 1]{Recht11}, or also Lemma 7.2 of \cite{Liu10_2}). 

It follows that \[\|Z\|_*> \|A\|_*= \|V_rV_r^*\|_*,\]
for any $B\neq0$.  Hence $Z=V_rV_r^*$ is the unique minimizer of problem \eqref{eqn:lrr:nuclear:self}.
\end{proof}

\begin{corollary}
Let $\W$ be as in Theorem \ref{THMWeiLin}, and let $\W=\W R^\dagger R$ be a factorization of $\W$ as in Theorem \ref{CURdecomp}.  Then $R^\dagger R=\W^\dagger \W=V_rV_r^*$ is the unique solution to the minimization problem \eqref{eqn:lrr:nuclear:self}.
\end{corollary}
 \begin{proof}
 Combine Corollary \ref{corollaryRR} and Theorem \ref{THMWeiLin}.
 \end{proof}

Let us note that the unique minimizer of problem \eqref{eqn:lrr:nuclear} is known for general dictionaries as the following result of Liu, Lin, Yan, Sun, Yu, and Ma demonstrates.

\begin{theorem}[\cite{Liu10_2}, Theorem 4.1]
Suppose that $A$ is a dictionary that linearly spans $\W$.  Then the unique minimizer of problem \eqref{eqn:lrr:nuclear} is \[Z = A^\dagger \W.\]
\end{theorem}

The following corollary is thus immediate from the CUR decomposition.

\begin{corollary}
If $\W$ has CUR decomposition $\W = CC^\dagger\W$ (where $R=\W$, hence $U=C$, in Theorem \ref{CURdecomp}), then $C^\dagger\W$ is the unique solution to
\[\min_{Z}\|Z\|_*\quad \textrm{s.t.} \quad \W=CZ.\]
\end{corollary} 

The above theorems and corollaries provide a theoretical bridge between the shape interaction matrix, CUR decomposition, and Low-Rank Representation.  In particular, in \cite{Wei11}, the authors observe that of primary interest is that while LRR stems from sparsity considerations \`{a} la compressed sensing, its solution in the noise free case in fact comes from matrix factorization, which is quite interesting.

\subsection{Basis Framework of \cite{Akram16}}

As a final note, the CUR framework proposed here gives a broader vantage point for obtaining similarity matrices than that of \cite{Akram16}.  Consider the following, which is the main result therein:

\begin{theorem}[\cite{Akram16}, Theorem 2]\label{THMAkram16}
Let $\W\in\mathbb{K}^{m\times n}$ be a rank $r$ matrix whose columns are drawn from $\mathcal{U}$.  Suppose $\W=BP$ where the columns of $B$ form a basis for the column space of $\W$ and the columns of $B$ lie in $\mathcal{U}$ (but are not necessarily columns of $\W$).  If $Q=\text{abs}(P^*P)$ or $Q=\text{bin}(P^*P)$, then $\Xi_\W=Q^{d_{\max}}$ is a similarity matrix for $\W$.
\end{theorem}

At first glance, Theorem \ref{THMAkram16} is on the one hand more specific than Theorem \ref{theoremsimilarity} since the columns of $B$ must be a basis for the span of $\W$, whereas $C$ may have some redundancy (i.e. the columns form a frame).  On the other hand, Theorem \ref{THMAkram16} seems more general in that the columns of $B$ need only come from $\mathcal{U}$, but are not forced to be columns of $\W$ as are the columns of $C$.  However, one need only notice that if $\W=BP$ as in the theorem above, then defining $\tilde\W=[\W \;\; B]$ gives rise to the CUR decomposition $\tilde\W = BB^\dagger \tilde\W$.  But clustering the columns of $\tilde\W$ via Theorem \ref{theoremsimilarity} automatically gives the clusters of $\W$.

\section{Proofs}\label{SECProofs}

Here we enumerate the proofs of the results in Section \ref{SECCURSeg}, beginning with some lemmata. 

\subsection{Some Useful Lemmata}

The first lemma follows immediately from the definition of independent subspaces.

\begin{lemma}
\label{lemmaUdagger}
Suppose $U = [U_1\dots U_M]$ where each $U_i$ is a submatrix whose columns come from independent subspaces of $\mathbb{K}^m$.  Then we may write 
\[U = [B_1 \; \dots \; B_M]\begin{bmatrix} V_1 & 0 & \dots & 0\\ 0 & V_2 & \dots & 0\\ \vdots & \vdots & \ddots & \vdots\\ 0 & \dots & 0 & V_M\\ \end{bmatrix}.\]
where the columns of  $B_i$ form a basis for the column space of $U_i$. 
\end{lemma}

The next lemma is a special case of \cite[Lemma 1]{Akram16}.

\begin{lemma}\label{LEMGeneric}
Suppose that $U\in\mathbb{K}^{m\times n}$ has columns which are generic for the subspace $S$ of $\mathbb{K}^m$ from which they are drawn.  Suppose $P\in\mathbb{K}^{r\times m}$ is a row selection matrix such that $\rank(PU)=\rank(U)$.  Then the columns of $PU$ are generic.
\end{lemma}

\begin{lemma}\label{LEMIndependent}
Suppose that $U = [U_1 \;\dots\; U_M]$, and  that each $U_i$ is a submatrix whose columns come from independent subspaces $S_i$, $i=1,\dots,M$ of $\mathbb{K}^m$, and that the columns of $U_i$ are generic for $S_i$.  Suppose that $P\in\mathbb{K}^{r\times m}$ with $r\geq\rank(U)$ is a row selection matrix such that $\rank(PU)=\rank(U)$. Then the subspaces $P(S_i)$, $i=1,\dots,M$ are independent.
\end{lemma}

\begin{proof}
Let $S=S_1+\dots+S_M$.  Let $d_i=\dim (S_i)$, and $d=\dim (S)$. From the hypothesis, we have that $\rank (PU_i)=d_i$, and $\rank (PU)=d=\sum\limits_{i=1}^M d_i$. Therefore, there are $d$ linearly independent vectors for $P(S)$ in the columns of $PU$.  Moreover, since  $PU=[PU_1,\dots,PU_M]$, there exist  $d_i$ linearly independent vectors from the columns of $PU_i$ for $P(S_i)$.  Thus, $\dim P(S)=d=\sum_id_i=\sum_i\dim P(S_i)$, and the conclusion of the lemma follows.\end{proof}

The next few facts which will be needed come from basic graph theory.  Suppose $G=(V,E)$ is a finite, undirected, weighted graph with vertices in the set $V=\{v_1,\dots,v_k\}$ and edges $E$.  The geodesic distance between two vertices $v_i,v_j\in V$ is the length (i.e. number of edges) of the shortest path connecting $v_i$ and $v_j$, and the {\em diameter} of the graph $G$ is the maximum of the pairwise geodesic distances of the vertices.  To each weighted graph is associated an adjacency matrix, $A$, such that $A(i,j)$ is nonzero if there is an edge between the vertices $v_i$ and $v_j$, and $0$ if not.  We call a graph {\em positively weighted} if $A(i,j)\geq0$ for all $i$ and $j$.  From these facts, we have the following lemma.

\begin{lemma}\label{LEMGraphPowerConnected}
Let $G$ be a finite, undirected, positively weighted graph with diameter $r$ such that every vertex has a self-loop, and let $A$ be its adjacency matrix.  Then $A^r(i,j)>0$ for all $i,j$.  In particular, $A^r$ is the adjacency matrix of a fully connected graph.
\end{lemma}

\begin{proof}
See \cite[Corollary 1]{Akram16}.
\end{proof}

The following lemma connects the graph theoretic considerations with the subspace model described in the opening.

\begin{lemma} \label{lemmageneric2} Let $V=\{p_1,\dots,p_{N}\}$ be a set of generic vectors that represent data from a subspace $S$ of dimension $r$ and $N>r\ge1$. Let $Q$ be as in Theorem \ref{theoremsimilarity} for the case $\U=S$.  Finally, let $G$ be the graph whose vertices are $p_i$ and whose edges are those $p_ip_j$  such that $Q(i,j)>0$. Then $G$ is connected, and has diameter at most $r$. Moreover, $Q^r$ is the adjacency matrix of a fully connected graph.
\end{lemma} 

\begin{proof}
The proof is essentially contained in \cite[Lemmas 2 and 3]{Akram16}, but for completeness is presented here.

First, to see that $G$ is connected, first consider the base case when $N=r+1$, and let $C$ be a non-empty set of vertices of a connected component of $G$.  Suppose by way of contradiction that $C^C$ contains $1\leq k\leq r$ vertices.  Then since $N=r+1$, we have that $|C^C|\leq r$, and hence the vectors $\{p_j\}_{j\in C^C}$ are linearly independent by the generic assumption on the data.  On the other hand, $|C|\leq r$, so the vectors $\{p_i\}_{i\in C}$ are also linearly independent. But by construction $\bracket{p,q}=0$ for any $p\in C$ and $q\in C^C$, hence the set $\{p_i\}_{i\in C\cup C^C}$ is a linearly independent set of $r+1$ vectors in $S$ which contradicts the assumption that $S$ has dimension $r$.  Consequently, $C^C$ is empty, i.e. $G$ is connected.

For generic $N>r$, suppose $p\neq q$ are arbitrary elements of $V$.  It suffices to show that there exists a path connecting $p$ to $q$.  Since $N>r$ and $V$ is generic, there exists a set $V_0\subset V\setminus\{p,q\}$ of cardinality $r-1$ such that $\{p,q\}\cup V_0$ is a set of linearly independent vectors in $S$.  This is a subgraph of $r+1$ vertices of $G$ which satisfies the conditions of the base case when $N=r+1$, and so this subgraph is connected.  Hence, there is a path connecting $p$ and $q$, and since these were arbitrary, we can conclude that $G$ is connected.

Finally, note that in the proof of the previous step for general $N>r$, we actually obtain that there is a path of length at most $r$ connecting any two arbitrary vertices $p,q\in V$.  Thus, the diameter of $G$ is at most $r$.  The moreover statement follows directly from Lemma \ref{LEMGraphPowerConnected}, and so the proof is complete.
\end{proof}

\subsection{Proof of Theorem \ref{theoremframes}}

Without loss of generality, we assume that $\W$ is such that $\W=[\W_1 \dots \W_M]$ where $\W_i$ is an $m\times n_i$ matrix for each $i=1,\dots,M$ and $\sum\limits_i^M n_i=n$, and the vector columns of $\W_i$ come from the subspace $S_i$. Under this assumption, we will show that $Y$ is a block diagonal matrix.

Let $P$ be the row selection matrix such that $P\W=R$; in particular, note that $P$ maps $\R^m$ to $\R^s$, and that because of the structure of $\W$, we may write $R = [R_1 \dots R_M]$ where the columns of $R_i$ belong to the subspace $\tilde S_i:=P(S_i)$.  Note also that the columns of each $R_i$ are generic for the subspace $\tilde S_i$ on account of Lemma \ref{LEMGeneric}, and that the subspaces $\tilde S_i$ are independent by Lemma \ref{LEMIndependent}.  Additionally, since $U$ consists of certain columns of $R$, and $\rank(U)=\rank(R)=\rank(\W)$ by assumption, we have that $U=[U_1\;\dots\; U_M]$ where the columns of $U_i$ are in $\tilde S_i$.

Next, recall from the proof of the CUR decomposition that  $Y=U^\dagger R$ is a solution to $R=UX$; thus $R = UY$.  Suppose that $r$ is a column of $R_1$, and let $y = [y_1\; y_2\; \dots\; y_M]^*$ be the corresponding column of $Y$ so that $r=Uy$.  Then we have that $r = \sum_{j=1}^M U_jy_j$, and in particular, since $r$ is in $R_1$,
\[(U_1y_1-r) + U_2y_2+\dots +U_My_M=0,\]
whereupon the independence of the subspaces $\tilde S_i$ implies that $U_1y_1=r$ and $U_iy_i=0$ for every $i=2,\dots,M$.  On the other hand, note that $\tilde y = [y_1\; 0 \;\dots \;0]^*$ is another solution; i.e. $r=U\tilde y$.  Recalling that $U^\dagger y$ is the minimal-norm solution to the problem $r=Uy$, we must have that
\[\|y\|_2^2=\sum_{i=1}^M\|y_i\|_2^2\leq\|\tilde y\|_2^2 = \|y_1\|_2^2.\]
Consequently, $y=\tilde y$, and it follows that $Y$ is block diagonal by applying the same argument for columns of $R_i$, $i=2,\dots,M$.

\hfill\qed

\subsection{Proof of Theorem \ref{theoremsimilarity}}

Without loss of generality, on account of Theorem \ref{theoremframes}, we may assume that $Y$ is block diagonal as above.  Then we first demonstrate that each block $Y_i$ has rank $d_i=\dim(S_i)$ and has columns which are generic.  Since $R_i = U_iY_i$, and $\rank(R_i)=\rank(U_i)=d_i$, we have $\rank(Y_i)\geq d_i$ since the rank of a product is at most the minimum of the ranks.  On the other hand, since $Y_i=U^\dagger R_i$, $\rank(Y_i)\leq \rank(R_i)=d_i$, whence $\rank(Y_i)=d_i$.  To see that the columns of each $Y_i$ are generic, let $y_1,\dots, y_{d_i}$ be $d_i$ columns in $Y_i$ and suppose there are constants $c_1,\dots,c_{d_i}$ such that $\sum_{j=1}^{d_i} c_jy_j=0.$  It follows from the block diagonal structure of $Y$ that
\[0 = U_i\left(\sum_{j=1}^{d_i}c_jy_j\right)=\sum_{j=1}^{d_i}c_jU_iy_j = \sum_{j=1}^{d_i}c_jr_j,\] where $r_j$, $j=1,\dots,d_i$ are the columns of $R_i$.  Since the columns of $R_i$ are generic by Lemma \ref{LEMGeneric}, it follows that $c_i=0$ for all $i$, whence the columns of $Y_i$ are generic.

Now $Q=Y^*Y$ is an $n\times n$ block diagonal matrix whose blocks are given by $Q_i=Y_i^*Y_i$, $i=1,\dots,M$, and we may consider each block as the adjacency matrix of a graph as prescribed in Lemma \ref{LEMGraphPowerConnected}.  Thus from the conclusion there, each block gives a connected graph with diameter $d_i$, and so $Q^{d_{\max}}$ will give rise to a graph with $M$ distinct fully connected components, where the graph arising from $Q_i$ corresponds to the data in $\W$ drawn from $S_i$.  Thus $Q^{d_{\max}}$ is indeed a similarity matrix for $\W$. 

\hfill\qed

\subsection{Proofs of Corollaries}\label{SECCor}

\begin{proof}[Proof of Corollary \ref{corollarySIM}]
That $\Xi_W$ is a similarity matrix follows directly from Theorem \ref{theoremsimilarity}.   To see the moreover statement, recall that $\W^\dagger=V_r\Sigma_r^\dagger U_r^*$, whence $\W^\dagger \W = V_r\Sigma_r^\dagger U_rU_r^*\Sigma_r V_r^*=V_r\Sigma_r^\dagger\Sigma_r V_r^* =V_rV_r^*$. 
\end{proof}

\begin{proof}[Proof of Corollary \ref{corollaryRR}]
By Lemma \ref{lemmaUdagger}, we may write $\W=BZ$, where $Z$ is block diagonal, and $B = [B_1 \;\dots\; B_M]$ with the columns of $B_i$ being a basis for $S_i$.  Let $P$ be the row-selection matrix which gives $R$, i.e. $R=P\W$.  Then $R=PBZ$.  The columns of $B$ are linearly independent (and likewise for the columns of $PB$ by Lemma \ref{LEMIndependent}), whence $\W^\dagger = Z^\dagger B^\dagger$, and $R^\dagger = Z^\dagger (PB)^\dagger$.  Moreover, linear independence of the columns also implies that $B^\dagger B$ and $(PB)^\dagger PB$ are both identity matrices of the appropriate size, whereby
\[ R^\dagger R = Z^\dagger (PB)^\dagger PB Z = Z^\dagger Z = Z^\dagger B^\dagger BZ =\W^\dagger \W,\] which completes the proof (note that the final identity $\W^\dagger \W=V_rV_r^*$ follows immediately from Corollary \ref{corollarySIM}).
\end{proof}

\section{Experimental Results}\label{SECExperiments}

\subsection{A proto-algorithm for noisy subspace data}

Up to this point, we have solely considered the noise-free case, when $\W$ contains uncorrupted columns of data lying in a union of subspaces satisfying the assumptions in Section \ref{SECAssumptions}.  We now depart from this to consider the more interesting case when the data is noisy as is typical in applications.  Particularly, we assume that a data matrix $\W$ is a small perturbation of an ideal data matrix, so that the columns of $\W$ approximately come from a union of subspaces; that is, $w_i = u_i+\eta_i$ where $u_i\in \U$ and $\eta_i$ is a small noise vector.  For our experimentation, we limit our considerations to the case of motion data.  

In motion segmentation, one begins with a video, and some sort of algorithm which extracts features on moving objects and tracks the positions of those features over time.  At the end of the video, one obtains a data matrix of size $2F\times N$ where $F$ is the number of frames in the video and $N$ is the number of features tracked.  Each vector corresponds to the trajectory of a fixed feature, i.e. is of the form $(x_1,y_1,\dots,x_F,y_F)^T$, where $(x_i,y_i)$ is the position of the feature at frame $1\leq i\leq F$.  Even though these trajectories are vectors in a high dimensional ambient space $\R^{2F}$, it is known that the trajectories of all feature belonging to the same rigid body lie in a subspace of dimension $4$ \cite{Kanatani02}. Consequently, motion segmentation is a suitable practical problem to tackle in order to verify the validity of the proposed approach.

\begin{algorithm}
\SetAlgorithmName{Proto-algorithm}{proto}{}
\SetAlgoRefName{}
\caption{CUR Subspace Clustering}
  \KwIn{A data matrix $\W=[w_1 \dotsb w_N] \in \mathbb{K}^{m \times n}$, expected number of subspaces M, and number of trials $k$.}
  \KwOut{Subspace (cluster) labels}

\For{ i = 1 to k}{
  Find approximate CUR factorization of $\W$, $\W\approx C_iU_i^\dagger R_i$\\
  $Y_i = U_i^\dagger R_i$\\
  Threshold $Y_i$\\
  Compute $\Xi_\W^{(i)} = Y_i^*Y_i$\\
  Enforce known connections\\
}
$\Xi_\W = \text{abs}(\text{median}(\Xi_\W^{(1)},\dots,\Xi_\W^{(k)}))$\\
Cluster the columns of $\Xi_\W$\\
\Return cluster labels
\label{protoalgorithm}
\end{algorithm}

It should be evident that the reason for the term proto-algorithm is that several of the steps admit some ambiguity in their implementation.  For instance, in line 2, how should one choose a CUR approximation to $\W$?  Many different ways of choosing columns and rows have been explored, e.g. by Drineas et. al. \cite{Drineas06III}, who choose rows and columns randomly according to a probability distribution favoring large columns or rows, respectively.  On the other hand, Demanet and Chiu \cite{Demanet13} show that uniformly selecting columns and rows can also perform quite well.  In our subsequent experiments, we choose to select rows and columns uniformly at random.  

There is yet more ambiguity in this step though, in that one has the flexibility to choose different numbers of columns and rows (as long as each number is at least the rank of $\W$).  Our choice of the number of rows and columns will be discussed in the next subsections.  

Thirdly, the thresholding step in line 4 of the proto-algorithm deserves some attention.  In experimentation, we tried a number of thresholds, many of which were based on the singular values of the data matrix $\W$.  However, the threshold that worked the best for motion data was a volumetric one based on the form of $Y_i$ guaranteed by Theorem \ref{theoremframes}.  Indeed, if there are $M$ subspaces each of the same dimension, then $Y_i$ should have precisely a proportion of $1-\frac1M$ of its total entries being $0$.  Thus our thresholding function orders the entries in descending order and keeps the top $(1-\frac1M)\times$ (total \# of entries), and sets the rest to 0.

Line 7 is simply a way to use any known information about the data to assist in the final clustering.  Typically, no prior information is assumed other than the fact that we obviously know that $w_i$ is in the same subspace as itself.  Consequently here, we force the diagonal entries of $\Xi_i$ to be $1$ to emphasize that this connection is guaranteed.

The reason for the appearance of an averaging step in line 8 is that, while a similarity matrix formed from a single CUR approximation of a data matrix may contain many errors, this behavior should average out to give the correct clusters.  The entrywise median of the family $\{\Xi_\W^{(i)}\}_{i=1}^k$ is used rather than the mean to ensure more robustness to outliers.  Note that a similar method of averaging CUR approximations was evaluated for denoising matrices in \cite{Sekmen2017matrix}; additional methods of matrix factorizations in matrix (or image) denoising may be found in \cite{Muhammad18}, for example.

Note also that we do not take powers of the matrix $Y_i^* Y_i$ as suggested by Theorem \ref{theoremsimilarity}.  The reason for this is that when noise is added to the similarity matrix, taking the matrix product multiplicatively enhances the noise and greatly degrades the performance of the clustering algorithm. There is good evidence that taking elementwise products of a noisy similarity matrix can improve performance \cite{Schnorr09,Ji15}.  This is really a different methodology than taking matrix powers, so we leave this discussion until later.

Finally, the clustering step in line 10 is there because at that point, we have a non-ideal similarity matrix $\Xi_\W$, meaning that it will not exactly give the clusters because there are few entries which are actually $0$.  However, each column should ideally have large $(i,j)$ entry whenever $w_i$ and $w_j$ are in the same subspace, and small (in absolute value) entries whenever $w_i$ and $w_j$ are not.  This situation mimics what is done in Spectral Clustering, in that the matrix obtained from the first part of the Spectral Clustering algorithm does not have rows which lie exactly on the coordinate axes given by the eigenvectors of the graph Laplacian; however hopefully they are small perturbations of such vectors and a basic clustering algorithm like $k$--means can give the correct clusters.  We discuss the performance of several different basic clustering algorithms used in this step in the sequel.

In the remainder of this section, the proto-algorithm above is investigated by first considering its performance on synthetic data, whereupon the initial findings are subsequently verified using the motion segmentation dataset known as Hopkins155 \cite{Vidal07}.

\subsection{Simulations using Synthetic Data}

A set of simulations are designed using synthetic data. In order for the results to be comparable to that of the motion segmentation case presented in the following section, the data is constructed in a similar fashion.  Particularly, in the synthetic experiments, data comes from the union of independent $4$ dimensional subspaces of $\R^{300}$.  This corresponds to a feature being tracked for $5$ seconds in a $30$~fps video stream. Two cases similar to the ones in Hopkins155 dataset are investigated for increasing levels of noise. In the first case, two subspaces of dimension 4 are randomly generated, while in the second case three subspaces are generated. In both cases, the data is randomly drawn from the unit ball of the subspaces.  In both cases, the level of noise on $\W$ is gradually increased from the initial noiseless state to the maximum noise level. The entries of the noise are i.i.d.  $\mathcal{N}(0,\sigma^{2})$ random variables (i.e. with zero-mean and variance $\sigma^2$), where the variance increases as $\sigma = [0.000, ~0.001, ~0.010, ~0.030, ~0.050, ~0.075, ~0.10]$.

In each case, for each noise level, $100$ data matrices are randomly generated containing $50$ data vectors in each subspace. Once each data matrix $\W$ is formed, a similarity matrix $\Xi_{\W}$ is generated using the proto-algorithm for CUR subspace clustering.  The parameters used are as follows: in the CUR approximation step (line 2), we choose all columns and the expected rank number of rows.  Therefore, the matrix $Y$ of Theorem \ref{theoremsimilarity} is of the form $R^\dagger R$. The expected rank number of rows (i.e. the number of subspaces times 4) are chosen uniformly at random from $\W$, and it is ensured that $R$ has the expected rank before proceeding; the thresholding step (line 4) utilizes the volumetric threshold described in the previous subsection, so in Case 1 we eliminate half the entries of $Y$, while in Case 2 we eliminate $2/3$ of the entries; we set $k=25$, i.e. we compute $25$ separate similarity matrices for each of the $100$ data matrices and subsequently take the entrywise median of the 25 similarity matrices (line 8) -- extensive testing shows no real improvement for larger values of $k$, so this value was settled on empirically; finally, we utilized three different clustering algorithms in line 10: $k$--means, Spectral Clustering \cite{Luxburg07}, and Principal Coordinate Clustering (PCC) \cite{PCC}.  To spare unnecessary clutter, we only display the results of the best clustering algorithm in Figure \ref{fig:SyntheticCases_n1}, which turns out to be PCC, and simply state here that using Matlab's $k$--means algorithm gives very poor performance even at low noise levels, while Spectral Clustering gives good performance for low noise levels but degrades rapidly after about $\sigma=0.05$.  More specifically, $k$--means has a range of errors from $0-50$\% even for the $\sigma=0$ in Case 2, while having an average of $36$\% and $48$\% clustering error in Case 1 and 2, respectively for the maximal noise level $\sigma=0.1$.  Meanwhile, Spectral Clustering exhibited on average perfect classification up to noise level $\sigma=0.03$ in both cases, but jumped rapidly to an average of  12\% and 40\% error in Case 1 and 2, respectively for the case $\sigma=0.1$.  Figure \ref{fig:SyntheticCases_n1} below shows the error plots for both cases when utilizing PCC. For a full description of PCC the reader is invited to consult \cite{PCC}, but essentially it eliminates the use of the graph Laplacian as in Spectral Clustering, and instead uses the principal coordinates of the first few left singular vectors in the singular value decomposition of $\Xi_\W$.  Namely, if $\Xi_\W$ has skinny SVD of order $M$ (the number of subspaces) of $U_M\Sigma_M V_M^*$, then PCC clusters the rows of $\Sigma_M V_M^*$ using $k$--means.

Results for Case 1 and Case 2 using PCC are given in Figure~\ref{fig:SyntheticCases_n25}.  For illustration, Figure~\ref{fig:SyntheticCases_n1} shows the performance of a single CUR decomposition to form $\Xi_\W$.  As expected, the randomness involved in finding a CUR decomposition plays a highly nontrivial role in the clustering algorithm.  It is remarkable to note, however, the enormity of the improvement the averaging procedure gives for higher noise levels, as can be seen by comparing the two figures.

\begin{figure}[h!]
\centering
\begin{subfigure}{.45\textwidth}
  \centering
  \includegraphics[width=1\linewidth]{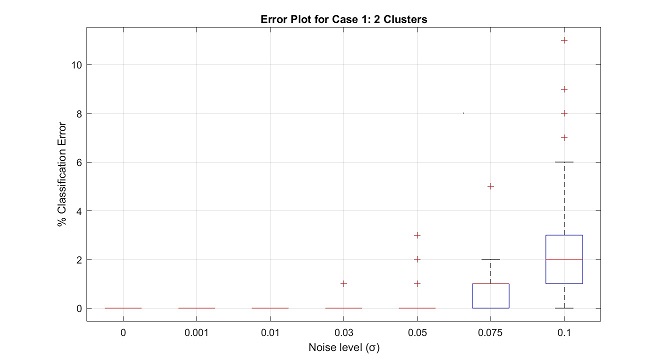}
  \caption{Case 1}
  \label{fig:SyntheticCase1n25}
\end{subfigure}
\begin{subfigure}{.45\textwidth}
  \centering
  \includegraphics[width=1\linewidth]{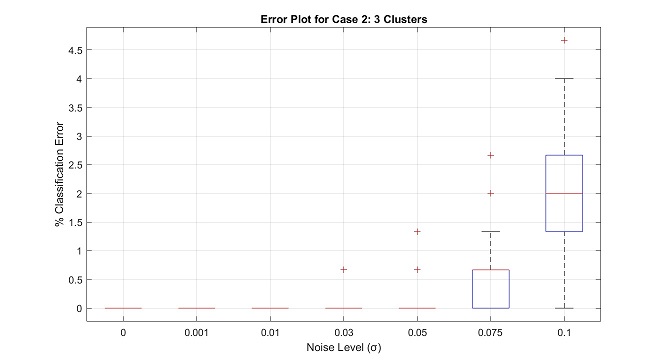}
  \caption{Case 2}
  \label{fig:SyntheticCase2n25}
\end{subfigure}
\caption{Synthetic Cases 1 and 2 for $\Xi_{\W}$ calculated using the median of $25$ CUR decompositions}
\label{fig:SyntheticCases_n25}
\end{figure}

\begin{figure}[h!]
\centering
\begin{subfigure}{.45\textwidth}
  \centering
  \includegraphics[width=1\linewidth]{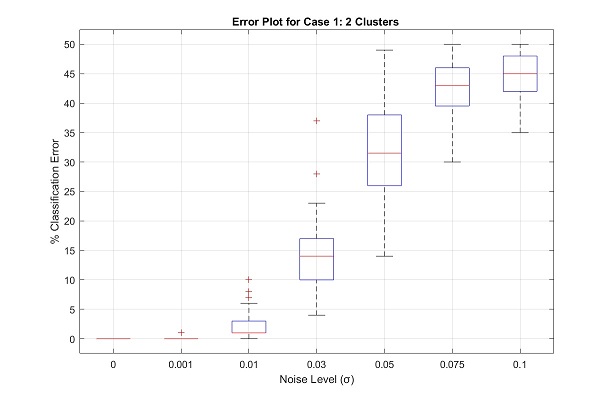}
  \caption{Case 1}
  \label{fig:SyntheticCase1}
\end{subfigure}
\begin{subfigure}{.45\textwidth}
  \centering
  \includegraphics[width=1\linewidth]{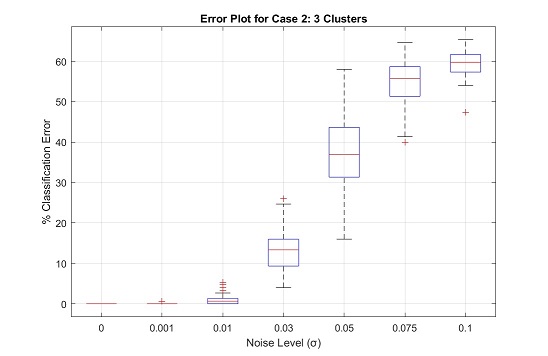}
  \caption{Case 2}
  \label{fig:SyntheticCase2}
\end{subfigure}
\caption{Synthetic Cases 1 and 2 for $\Xi_{\W}$ calculated using a single CUR decomposition}
\label{fig:SyntheticCases_n1}
\end{figure}

\subsection{Motion Segmentation Dataset: Hopkins155}

The Hopkins155 motion dataset contains 155 videos which can be broken down into several categories: checkerboard sequences where moving objects are overlaid with a checkerboard pattern to obtain many feature points on each moving object, traffic sequences, and articulated motion (such as people walking) where the moving body contains joints of some kind making the 4-dimensional subspace assumption on the trajectories incorrect.  Associated with each video is a data matrix giving the trajectories of all features on the moving objects (these features are fixed in advance for easy comparison).  Consequently, the data matrix is unique for each video, and the ground truth for clustering is known {\em a priori}, thus allowing calculation of the clustering error, which is simply the percentage of incorrectly clustered feature points.  An example of a still frame from one of the videos in Hopkins155 is shown in Figure \ref{FIG:Hopkins}.  Here we are focused only on the clustering problem for motion data; however, there are many works on classifying motions in video sequences which are of a different flavor, e.g. \cite{ChoMotion,Khan18,Kirby18}.

\begin{figure}
\centering
\includegraphics[scale=0.3]{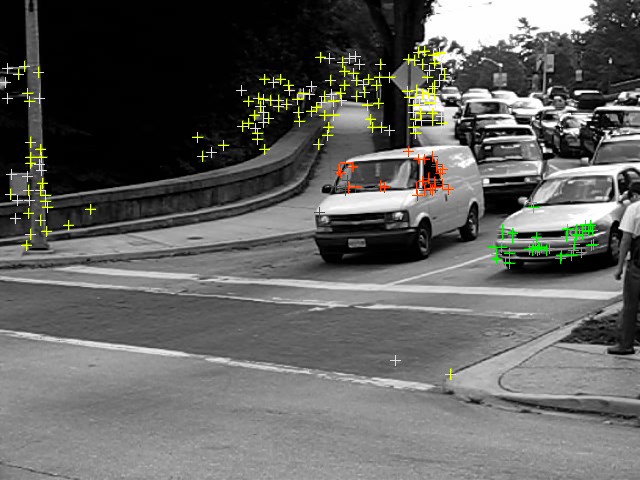}\quad \includegraphics[scale=0.3]{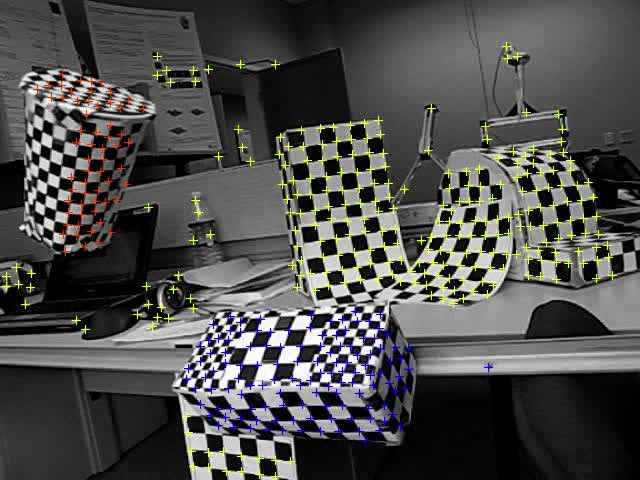}
\caption{ Example stills from a car sequence (left) and checkerboard sequence (right) from the Hopkins155 motion dataset.}\label{FIG:Hopkins}
\end{figure}

As mentioned, clustering performance using CUR decompositions is tested using the Hopkins155 dataset. 
In this set of experiments, we use essentially the same parameters as discussed in the previous subsection when testing synthetic data.  That is, we use CUR approximations of the form $\W R^\dagger R$ where exactly the expected rank number of rows are chosen uniformly at random, the volumetric threshold is used, and we take the median of 50 similarity matrices for each of the 155 data matrices (we use $k=50$ here instead of 25 to achieve a more robust performance on the real dataset).  Again, PCC is favored over $k$--means and Spectral Clustering due to a dramatic increase in performance. 
 
It turns out that for real motion data, CUR yields better overall results than SIM, the other pure similarity matrix method.  This is reasonable given the flexibility of the CUR decomposition.  Finding several similarity matrices and averaging them has the effect of averaging out some of the inherent noise in the data.  The classification errors for our method as well as many benchmarks appear in Tables \ref{tab:twomotion}--\ref{tab:overall}.  Due to the variability of any single trial over the Hopkins155 dataset, the CUR data presented in the tables is the average of 20 runs of the algorithm over the entire dataset.  The purpose of this work is not to fine tune CUR's performance on the Hopkins155 dataset; nonetheless, the results using this simple method are already better than many of those in the literature.

\begin{table}[h]
\centering
\scriptsize
		\caption{\% classification errors for sequences with two motions.}
		\label{tab:twomotion}
		\begin{tabular}{||c||cccccccc||cc||}
					\hline
				\textbf{\em Checker (78)} &GPCA&LSA&RANSAC&MSL&ALC&SSC-B&SSC-N&NLS&SIM&CUR\\
					\hline \hline
				Average & 6.09\%    & 2.57\%& 6.52\% &4.46\% & 1.55\%   & 0.83\%  & 1.12\% & 0.23\% & 3.52\% & 0.94\% \\
				Median & 1.03\%    & 0.27\%& 1.75\% &0.00\% & 0.29\%   & 0.00\%  & 0.00\% & 0.00\%  & 0.80\% & 0.00\% \\
					\hline \hline
				\textbf{\em Traffic (31)}   &GPCA&LSA&RANSAC&MSL&ALC&SSC-B&SSC-N&NLS&SIM&CUR\\
					\hline \hline
				Average & 1.41\%     & 5.43\%& 2.55\% &2.23\% & 1.59\%   & 0.23\%  & 0.02\% & 1.40\% & 8.46\% & 1.08\% \\
				Median & 0.00\%    & 1.48\%& 0.21\% &0.00\% & 1.17\%   & 0.00\%  & 0.00\% & 0.00\% & 1.95\% & 0.05\% \\
					\hline
						\hline \hline
				\textbf{\em Articulated (11)} &GPCA&LSA&RANSAC&MSL&ALC&SSC-B&SSC-N&NLS&SIM&CUR\\
					\hline \hline
				Average & 2.88\%    & 4.10\%& 7.25\% &7.23\% & 10.70\%   & 1.63\%  & 0.62\% & 1.77\%  & 6.00\% & 6.36\% \\ 
				Median & 0.00\%   & 1.22\%& 2.64\% &0.00\% & 0.95\%   & 0.00\%  & 0.00\%&  0.88\% & 0.43\% & 0.00\% \\
				\hline
					\hline \hline
				\textbf{\em All (120 seq)} &GPCA&LSA&RANSAC&MSL&ALC&SSC-B&SSC-N&NLS&SIM&CUR\\
					\hline \hline
				Average & 4.59\%     & 3.45\%& 5.56\% &4.14\% & 2.40\%   & 0.75\%  & 0.82\%&  0.57\%  & 5.03\% & 1.47\% \\
				Median & 0.38\%   & 0.59\%& 1.18\% &0.00\% & 0.43\%   & 0.00\%  & 0.00\%&  0.00\% & 1.07\% & 0.00\% \\
				\hline
		\end{tabular}
\end{table}

\begin{table}[h]
\centering
\scriptsize
		\caption{\% classification errors for sequences with three motions.}
		\label{tab:threemotion}
		\begin{tabular}{||c||cccccccc||cc||}
					\hline
				\textbf{\em Checker (26)} &GPCA&LSA&RANSAC&MSL&ALC&SSC-B&SSC-N&NLS&SIM&CUR \\
					\hline \hline
				Average & 31.95\%    & 5.80\%& 25.78\% &10.38\% & 5.20\%   & 4.49\%  & 2.97\%&  0.87\%  & 8.26\% & 3.25\% \\
				Median & 32.93\%    & 1.77\%& 26.00\% &4.61\% & 0.67\%   & 0.54\%  & 0.27\%&  0.35\% & 2.16\% & 1.05\% \\ 
					\hline \hline
				\textbf{\em Traffic (7)}   &GPCA&LSA&RANSAC&MSL&ALC&SSC-B&SSC-N&NLS&SIM&CUR\\
					\hline \hline
				Average & 19.83\%  & 25.07\%& 12.83\% &1.80\% & 7.75\%   & 0.61\%  & 0.58\%&  1.86\%  & 16.59\% & 3.57\% \\
				Median & 19.55\%    & 23.79\%& 11.45\% &0.00\% & 0.49\%   & 0.00\%  & 0.00\%& 1.53\%  & 10.33\% & 0.72\% \\
					\hline
						\hline \hline
				\textbf{\em Articulated (2)} &GPCA&LSA&RANSAC&MSL&ALC&SSC-B&SSC-N&NLS&SIM&CUR\\
					\hline \hline
				Average & 16.85\%    & 7.25\%& 21.38\% &2.71\% & 21.08\%   & 1.60\%  & 1.60\%&  5.12\%  & 18.42\% & 8.80\% \\
				Median & 16.85\%   & 7.25\%& 21.38\% &2.71\% & 21.08\%   & 1.60\%  & 1.60\%&  5.12\%  & 18.42\% & 8.80\% \\
				\hline
					\hline \hline
				\textbf{\em All (35 seq)} &GPCA&LSA&RANSAC&MSL&ALC&SSC-B&SSC-N&NLS&SIM&CUR\\
					\hline \hline
				Average & 28.66\%     & 9.73\%& 22.94\% &8.23\% & 6.69\%   & 3.55\%  & 2.45\%&  1.31\%  & 10.51\% & 3.63\% \\
				Median & 28.26\%   & 2.33\%& 22.03\% &1.76\% & 0.67\%   & 0.25\%  & 0.20\%&  0.45\%  & 4.46\% & 0.92\% \\
				\hline
		\end{tabular}
\end{table}
\begin{table}[h!]
\centering
\scriptsize
		\caption{\% classification errors for all sequences.}
		\label{tab:overall}

		\begin{tabular}{||c||cccccccc||cc||}
					\hline
				\textbf{\em All (155 seq)} &GPCA&LSA&RANSAC&MSL&ALC&SSC-B&SSC-N&NLS&SIM&CUR\\
					\hline \hline
				Average & 10.34\%     & 4.94\%& 9.76\% &5.03\% & 3.56\%   & 1.45\%  & 1.24\%&  0.76\%  & 6.26\% & 1.96\% \\
				Median & 2.54\%   & 0.90\%& 3.21\% &0.00\% & 0.50\%   & 0.00\%  & 0.00\%&  0.20\% & 1.33\% & 0.07\% \\
				\hline
		\end{tabular}
\end{table}

To better compare performance, we timed the CUR-based method described above in comparison with some of the benchmark timings given in \cite{Vidal07}.  For direct comparison, we ran all of the algorithms on the same machine, which was a mid-2012 Macbook Pro with 2.6GHz Intel Core i7 processor and 8GB of RAM.  The results are listed in Table \ref{TAB:Time}.  We report the times for various values of $k$ as in Step 1 of the Proto-algorithm, indicating how many CUR approximations are averaged to compute the similarity matrix.

\begin{table}[h]
\centering
\scriptsize
		\caption{Run times (in s) over the entire Hopkins155 dataset for various algorithms.}
		\label{TAB:Time}
\begin{tabular}{||c|| c c c || c c c ||}
\hline
\textbf{Algorithm} & GPCA & LSA & RANSAC & CUR ($k=25$) & CUR ($k=50$) & CUR ($k=75$)\\ \hline \hline
Time (s) & 3.80 & 361.77 & 3.39 & 71.96 & 202.43 & 392.26\\
\hline
\end{tabular}
\end{table} 

\subsection{Discussion}

It should be noted that after thorough experimentation on noisy data, using a CUR decomposition which takes all columns and exactly the expected rank number of rows exhibits the best performance.  That is, a decomposition of the form $\W=\W R^\dagger R$ performs better on average than one of the form $\W=CU^\dagger R$.  The fact that choosing more columns performs better when the matrix $\W$ is noisy makes sense in that any representation of the form $\W=CX$ is a representation of $\W$ in terms of the frame vectors of $C$.  Consequently, choosing more columns in the matrix $C$ means that we are adding redundancy to the frame, and it is well-known to be one of the advantages of frame representations that redundancy provides greater robustness to noise.  Additionally, we noticed experimentally that choosing exactly $r$ rows in the decomposition $\W R^\dagger R$ exhibits the best performance.  It is not clear as of yet why this is the case.

As seen in the tables above, the proposed CUR based clustering algorithm works dramatically better than SIM, but does not beat the best state-of-the-art obtained by the first and fourth author in \cite{Akram12}.  More investigation is needed to determine if there is a way to utilize the CUR decomposition in a better fashion to account for the noise.  Of particular interest may be to utilized convex relaxation method recently proposed in \cite{CURTPAMI}. 

One interesting note from the results in the tables above is that, while some techniques for motion clustering work better on the checkered sequences rather than traffic sequences (e.g. NLS, SSC, and SIM), CUR seems to be blind to this difference in the type of data.  It would be interesting to explore this phenomenon further to determine if the proposed method performs uniformly for different types of motion data.  We leave this task to future research.

As a final remark, we note that the performance of CUR decreases as $k$, the number of CUR decompositions averaged to obtain the similarity matrix, increases.  The error begins to level out beyond about $k=50$, whereas the time steadily increases as $k$ does (see Table \ref{TAB:Time}).  One can easily implement an adaptive way of choosing $k$ for each data matrix rather than holding it fixed.  To test this, we implemented a simple threshold by taking, for each $i$ in the Proto-algorithm, a temporary similarity matrix $\tilde{\Xi}_{\W}^{(i)}:=\text{abs(median}(\Xi_{\W}^{(1)},\dots\Xi_{\W}^{(i)}))$.  That is, $\tilde{\Xi}_{\W}^{(i)}$ is the median of all of the CUR similarity matrices produced thus far in the for loop.  We then computed $\|\tilde\Xi_{\W}^{(i)}-\tilde\Xi_{\W}^{(i-1)}\|_2$, and compared it to a threshold (in this case $0.01$).  If the norm was less than the threshold, then we stopped the for loop, and if not we kept going, setting an absolute cap of 100 on the number of CUR decompositions used.  We found that on average, 57 CUR decompositions were used, with a minimum of 37, a maximum of 100 (the threshold value), and a standard deviation of 13.  Thus it appears that a roughly optimal region for fast time performance and good clustering accuracy is around $k=50$ to $k=60$ CUR decompositions.

\subsection{Robust CUR Similarity Matrix for Subspace Clustering}

We now turn to a modification of the Proto-algorithm discussed above.  One of the primary reasons for using the CUR decomposition as opposed to the shape interaction matrix ($V_rV_r^*$) is that the latter is not robust to noise.  However, Ji, Salzmann, and Li \cite{Ji15} proposed a robustified version of SIM, called RSIM.  The key feature of their algorithm is that they do not enforce the clustering rank beforehand, but they find a range of possible ranks, and make a similarity matrix for each rank in this range, and perform Spectral Clustering on a modification of the SIM.  Then, they keep the clustering labels from the similarity matrix for the rank $r$ which minimizes the associated minCut quantity on the graph determined by the similarity matrix.  

Recall given a weighted, undirected graph $G=(V,E)$, and a partition of its vertices, $\{A_1,\dots,A_k\}$, the Ncut value of the partition is
\[ \text{Ncut}_k(A_1,\dots,A_k) := \frac12 \sum_{i=1}^k\dfrac{W(A_i,A_i^C)}{\text{vol}(A_i)},\]
where $W(A, A^C) := \sum_{i\in A, j\in A^C}w_{i,j}$, where $w_{i,j}$ is the weight of the edge $(v_i,v_j)$ (defined to be $0$ if no such edge exists).  The RSIM method varies the rank $r$ of the SIM (i.e. the rank of the SVD taken), and minimizes the corresponding Ncut quantity over $r$.

The additional feature of the RSIM algorithm is that rather than simply taking $V_rV_r^*$ for the similarity matrix, they first normalize the rows of $V_r$ and then take elementwise power of the resulting matrix.  This follows the intuition of \cite{Schnorr09} and should be seen as a type of thresholding as in Step 4 of the Proto-algorithm.  For the full RSIM algorithm, consult \cite{Ji15}, but we present the CUR analogue that maintains most of the steps therein.

\begin{algorithm}
\caption{Robust CUR Similarity Matrix (RCUR)}
\label{ALG:RCUR}
  \KwIn{A data matrix $\W=[w_1 \dotsb w_N] \in \mathbb{K}^{m \times n}$, minimum rank $r_{\min}$ and maximum rank $r_{\max}$, number of trials $k$, and exponentiation parameter $\alpha$.}
  \KwOut{Subspace (cluster) labels and best rank $r_{\text{best}}$}
\For{$r = r_{\min}$ to $r_{\max}$}{
\For{ $i = 1$ to $k$}{
  Find approximate CUR factorization of $\W$, $\W\approx C_iU_i^\dagger R_i$\\
  $Y_i = U_i^\dagger R_i$\\
  Normalize columns of $Y_i$, call the resulting matrix $Y_i$\\
  Compute $\Xi_\W^{(i)} = Y_i^*Y_i$\\ 
}
$\Xi_\W = \text{abs}(\text{median}(\Xi_\W^{(1)},\dots,\Xi_\W^{(k)}))$\\
$(\Xi_{\W})_{i,j} = (\Xi_{\W})_{i,j}^\alpha$, i.e. take elementwise power of the similarity matrix\\
Cluster the columns of $\Xi_\W$\\
}
$r_{\text{best}} = \underset{r}{\text{argmin}}\; \text{Ncut}_r$\\
\Return cluster labels from trial $r_{\text{best}}$, and $r_{\text{best}}$.
\end{algorithm}	 

The main difference between Algorithm \ref{ALG:RCUR} and that of \cite{Ji15} is that in the latter, Steps 2--8 in Algorithm \ref{ALG:RCUR} are replaced with computing the thin SVD of order $r$, and normalizing the rows of $V_r$, and then setting $\Xi_{\W} = V_rV_r^*$.  In \cite{Ji15}, the Normalized Cuts clustering algorithm is preferred, which is also called Spectral Clustering with the symmetric normalized graph Laplacian \cite{Luxburg07}.  In our testing of RCUR, it appears that the normalization and elementwise powering steps interfere with the principal coordinate system in PC clustering; we therefore used Spectral Clustering as in the RSIM case.  Results are presented in Tables \ref{tab:twomotion-RCUR}--\ref{tab:overall-RCUR}.  Note that the values for RSIM may differ from those reported in \cite{Ji15}; there the authors only presented errors for all 2--motion sequences, all 3--motion sequences, and overall rather than considering each subclass (e.g. checkered or traffic).  We used the code from \cite{Ji15} for the values specified by their algorithm and report its performance in the tables below (in general, the values reported here are better than those in \cite{Ji15}).

\begin{table}[h]
\centering
\scriptsize
		\caption{\% classification errors for sequences with two motions.}
		\label{tab:twomotion-RCUR}
		\begin{tabular}{||c||ccccc||}
					\hline
				\textbf{\em Checker (78)} &SSC-B&SSC-N&NLS&RSIM&RCUR\\
					\hline \hline
				Average &  0.83\%  & 1.12\% & 0.23\% & 0.48\% &  \textbf{0.17\%}\\
				Median &  0.00\%  & 0.00\% & 0.00\%  & 0.00\% &  0.00\%\\
					\hline \hline
				\textbf{\em Traffic (31)}   &SSC-B&SSC-N&NLS&RSIM&RCUR\\
					\hline \hline
				Average &  0.23\%  & \textbf{0.02\%} & 1.40\% & 0.06\% &  0.09\%\\
				Median  & 0.00\%  & 0.00\% & 0.00\% & 0.00\% &  0.00\%\\
					\hline
						\hline \hline
				\textbf{\em Articulated (11)} &SSC-B&SSC-N&NLS&RSIM&RCUR\\
					\hline \hline
				Average &  1.63\%  & \textbf{0.62\%} & 1.77\%  & 1.43\% &  1.26\%\\ 
				Median &  0.00\%  & 0.00\%&  0.88\% & 0.00\% &  0.00\% \\
				\hline
					\hline \hline
				\textbf{\em All (120 seq)} &SC-B&SSC-N&NLS&RSIM&RCUR\\
					\hline \hline
				Average &  0.75\%  & 0.82\%&  0.57\%  & 0.46\% &  \textbf{0.25\%}\\
				Median & 0.00\%  & 0.00\%&  0.00\% & 0.00\% &  0.00\%\\
				\hline
		\end{tabular}
\end{table}

\begin{table}[h]
\centering
\scriptsize
		\caption{\% classification errors for sequences with three motions.}
		\label{tab:threemotion-RCUR}
		\begin{tabular}{||c||ccccc||}
					\hline
				\textbf{\em Checker (26)} &SSC-B&SSC-N&NLS&RSIM&RCUR \\
					\hline \hline
				Average &  4.49\%  & 2.97\%&  0.87\%  & 0.63\% &  \textbf{0.40\%}\\
				Median &  0.54\%  & 0.27\%&  0.35\% & 0.40\% &  \textbf{0.03\%}\\ 
					\hline \hline
				\textbf{\em Traffic (7)}   &SSC-B&SSC-N&NLS&RSIM&RCUR\\
					\hline \hline
				Average &  0.61\%  & \textbf{0.58\%}&  1.86\%  & 2.22\%  & 0.89\%\\
				Median &  0.00\%  & 0.00\%& 1.53\%  & 0.19\%  & 0.03\% \\ 
					\hline
						\hline \hline
				\textbf{\em Articulated (2)} &SSC-B&SSC-N&NLS&RSIM&RCUR\\
					\hline \hline
				Average &  \textbf{1.60\%}  & \textbf{1.60\%}&  5.12\%  & 18.95\%  & 4.81\%\\
				Median  & 1.60\%  & 1.60\%&  5.12\%  & 18.95\%  & 4.81\%\\
				\hline
					\hline \hline
				\textbf{\em All (35 seq)} &SSC-B&SSC-N&NLS&RSIM&RCUR\\
					\hline \hline
				Average & 3.55\%  & 2.45\%&  1.31\%  & 2.00\% &  \textbf{0.75\%}\\
				Median &  0.25\%  & 0.20\%&  0.45\%  & 0.43\% &  \textbf{0.00\%}\\
				\hline
		\end{tabular}
\end{table}
\begin{table}[h!]
\centering
\scriptsize
		\caption{\% classification errors for all sequences.}
		\label{tab:overall-RCUR}

		\begin{tabular}{||c||ccccc||}
					\hline
				\textbf{\em All (155 seq)} &SSC-B&SSC-N&NLS&RSIM&RCUR\\
					\hline \hline
				Average & 1.45\%  & 1.24\%&  0.76\%  & 0.81\%  & \textbf{0.36\%}\\
				Median & 0.00\%  & 0.00\%&  0.20\% & 0.00\% &  0.00\%\\
				\hline
		\end{tabular}
\end{table}

The results displayed in the table are obtained by choosing $k=50$ to be fixed (based on the analysis in the previous section), and taking CUR factorizations of the form $\W\approx C_i R_i^\dagger R_i$, where we choose $r$ rows to form $R_i$, where $r$ is given in Step 1 of the algorithm.  This is, by Corollary \ref{corollaryRR}, the theoretical equivalent of taking $V_rV_r^*$ as in RSIM.  Due to the randomness of finding CUR factorizations in Algorithm \ref{ALG:RCUR}, the algorithm was run 20 times and the average performance was reported in Tables \ref{tab:twomotion-RCUR}--\ref{tab:overall-RCUR}.  We note also that the standard deviation of the performance across the 20 trials was less than $0.5\%$ for all categories with the exception of the Articulated 3 motion category, in which case the standard deviation was large ($5.48\%$).  

As can be seen, the algorithm proposed here does not yield the best performance on all facets of the Hopkins155 dataset; however, it does achieve the best overall classification result to date with only $0.36\%$ average classification error.  As an additional note, running Algorithm \ref{ALG:RCUR} with only $k=10$ CUR factorizations used for each data matrix still yields relatively good results (total $2$--motion error of $0.41\%$, total $3$--motion error of $1.21\%$, and total error on all of Hopkins155 of $0.59\%$) while allowing for less computation time.  Preliminary tests suggest also that taking fewer rows in the CUR factorization step in Algorithm \ref{ALG:RCUR} works much better than in the version of the Proto-algorithm used in the previous sections (for instance, taking half of the available columns and $r$ rows still yields less than $0.5\%$ overall error).  However, the purpose of the current work is not to optimize all facets of Algorithm \ref{ALG:RCUR}, as much more experimentation needs to be done to determine the correct parameters for the algorithm, and it needs to be tested on a broad variety of datasets which approximately satisfy the union of subspaces assumption herein considered.  This task we leave to future work.

\section{Concluding Remarks}

The motivation of this work was truly the realization that the exact CUR decomposition of Theorem \ref{CURdecomp} can be used for the subspace clustering problem.  We demonstrated that, on top of its utility in randomized linear algebra, CUR enjoys a prominent place atop the landscape of solutions to the subspace clustering problem.  CUR provides a theoretical umbrella under which sits the known shape interaction matrix, but it also provides a bridge to other solution methods inspired by compressed sensing, i.e. those involving the solution of an $\ell_1$ minimization problem. Moreover, we believe that the utility of CUR for clustering and other applications will only increase in the future.  Below, we provide some reasons for the practical utility of CUR decompositions, particularly related to data analysis and clustering, as well as some future directions.

\underline{Benefits of CUR:}

\begin{itemize}
\item From a theoretical standpoint, the CUR decomposition of a matrix is utilizing a frame structure rather than a basis structure to factorize the matrix, and therefore enjoys a level of flexibility beyond something like the SVD.  This fact should provide utility for applications.

\item Additionally, a CUR decomposition remains faithful to the structure of the data.  For example, if the given data is sparse, then both $C$ and $R$ will be sparse, even if $U^\dagger$ is not in general.  In contrast, taking the SVD of a sparse matrix yields full matrices $U$ and $V$, in general.  

\item Often in obtaining real data, many entries may be missing or extremely corrupted.  In motion tracking, for example, it could be that some of the features are obscured from view for several frames.  Consequently, some form of matrix completion may be necessary.  On the other hand, a look at the CUR decomposition reveals that whole rows of a data matrix can be missing as long as we can still choose enough rows such that the resulting matrix $R$ has the same rank as $\W$.  
\end{itemize}

\underline{Future Directions}
 
\begin{itemize}
\item Algorithm \ref{ALG:RCUR} and other iterations of the Proto-algorithm presented here need to be further tested to determine the best way to utilize the CUR decomposition to cluster subspace data.  Currently, Algorithm \ref{ALG:RCUR} is somewhat heuristic, so a more fully understood theory concerning its performance is needed.  We note that some justification for the ideas of RSIM are given in \cite{Ji15}; however, the ideas there do not fully explain the outstanding performance of the algorithm. As commented above, one of the benefits of the Proto-algorithm discussed here is its flexibility, which provides a distinct advantage over SVD based methods.

\item  Another direction is to combine the CUR technique with sparse methods to construct algorithms that are strongly robust to noise and that allow  clustering when the data points are not drawn from a union of independent subspaces.   
\end{itemize}

\section*{Acknowledgements}
The research of  A. Sekmen, and A.B. Koku is supported by DoD Grant W911NF-15-1-0495. The research of A.~Aldroubi is supported by NSF Grant NSF/DMS 132099. The research of A.B. Koku is also supported by TUBITAK-2219-1059B191600150.  Much of the work for this article was done while K. Hamm was an assistant professor at Vanderbilt University.  In the final stage of the project, the research of K. Hamm was partially supported through the NSF TRIPODS project under grant CCF-1423411.

The authors also thank the referees for constructive comments which helped improve the quality of the manuscript.

\bibliographystyle{elsarticle-num}
\bibliography{sekmen}

\end{document}